\documentclass[conference]{IEEEtran}
\usepackage{times}

\usepackage[numbers, sort]{natbib}
\usepackage{multicol}
\usepackage[bookmarks=true]{hyperref}
\usepackage{amsmath}
\usepackage{amsfonts}
\usepackage{amssymb}  % assumes amsmath package installed
\usepackage{algorithm}
\usepackage{algpseudocode}
\usepackage{graphics} % for pdf, bitmapped graphics files
\usepackage{epsfig} % for postscript graphics files

\usepackage{color}
\usepackage{balance}
\usepackage{hyperref}

\newtheorem{lemma}{Lemma}

\usepackage{setspace}
\begin{document}

% paper title
\title{Data-Driven Measurement Models for Active Localization in Sparse Environments}

% You will get a Paper-ID when submitting a pdf file to the conference system
% \author{Author Names Omitted for Anonymous Review. Paper-ID 150}

\author{\authorblockN{Ian Abraham,
Anastasia Mavrommati,
Todd D. Murphey}
\authorblockA{Department of Mechanical Engineering \\
Northwestern University,
Evanston, IL 60208 }
}

\maketitle

\begin{abstract}
We develop an algorithm to explore an environment to generate a measurement model for use in future localization tasks.
Ergodic exploration with respect to the likelihood of a particular class of measurement (e.g., a contact detection measurement in tactile sensing) enables construction of the measurement model.
Exploration with respect to the information density based on the data-driven measurement model enables localization.
We test the two-stage approach in simulations of tactile sensing, illustrating that the algorithm is capable of identifying and localizing objects based on sparsely distributed binary contacts.
Comparisons with our method show that visiting low probability regions lead to acquisition of new information rather than increasing the likelihood of known information.
Experiments with the Sphero SPRK robot validate the efficacy of this method for collision-based estimation and localization of the environment.

\end{abstract}

\IEEEpeerreviewmaketitle

\section{Introduction}

Sensors such as cameras~\cite{georgakisRSS17objectdetection, hengAutRob15calibrationVisualSLAM, nuger2016multicamera} and LiDAR~\cite{schwarz2010lidar, rasshofer2005automotive} provide high quality data about an environment with remarkably low sensitivity to the location of the sensor.
As a result, the resulting sensing and perception are quite rich and are typically independent of the control.
However, some sensing modalities, such as contact/impact sensors and range sensors, only provide useful information when they are in the correct configuration or sequence of configurations\textemdash for these sensors, high quality data is sparse in the environment and must therefore be sought out by utilizing control authority.
In this paper, we investigate the use of active sensing, through means of ergodic control~\cite{miller2016ergodic, miller2015optimalrange, mavrommati2017eSAC}, to provide control authority that seeks out useful information from sensors for which informative measurements are sparse in the environment.

In biology, we see this combination of movement and sensing in rodents that use their whiskers for tactile sensing~\cite{guic1989rats,carvell1990biometric,hobbs2015spatiotemporal, mitchinson2007feedback}. 
These biological sensors are active sensors (i.e., sensing is done through movement) and provide a rich set of contact and force information. 
Notably, no two whiskers are the same mechanically\textemdash more so across rodents of the same species\textemdash and yet each rodent is able to use these biological sensors to sense and explore their environment.
Similarly, humans use their hands to grasp and to feel objects by actively moving their fingers across objects.
New objects are rapidly incorporated into memory without needing to understand and model the complex physical interactions between the hand and the object. 
We are interested in how one automates the generation of these models, and what principles extend to robotic settings.

The combination of environmentally sparse, unmodeled sensor-environment interactions with motion leads us to the topic of this work: equip a robot with the means of active sensing when the environment is unknown and the sensor interaction with the environment is spatially sparse.
In doing so, we seek to improve the utility of sparse sensors (e.g. contact sensors~\cite{meierTRO11probApproachTactileShape, mitchinson2007feedback, hobbs2015spatiotemporal}), and their interaction with the environment by synthesizing movement, both for model generation and later identification and localization.
We approach this problem of active sensing as a coverage problem where a robot's motion (and the time spent in regions of space) should coincide with an \emph{expectation of informative sensory input}.
Ergodic control policies~\cite{shellMult06ergodic, mathew2011metrics, miller2013trajectory,miller2016ergodic} are used to compute active sensing strategies where uncertainty about the environment, the robot's sensors, and the interaction is reasoned about as a dynamically evolving area coverage problem. 
Thus, the contribution of the work is the construction and use of online data-driven distributions that, when used to specify the ergodic control, enables a robot to compensate, through motion, for spatially sparse sensor information and unknown environments to estimate and localize the environment. 
Simulations illustrate the phases of the algorithm and experiments validate our approach for collision-based robot sensing.

The paper outline is as follows: Section~\ref{sec:related_work} discusses literature related to this work. 
Section~\ref{sec:problem-statement} introduces the problem statement that is addressed in the paper.
Section~\ref{sec:erg} introduces ergodicity and the ergodic control policy.
Section~\ref{sec:shape_estimation} then formulates a method for constructing data-driven measurement models of the environment.
Section~\ref{sec:shape_localization} then describes how to use these empirically determined measurement models for localization of the environment.
Simulated and experimental results are in Sections~\ref{sec:sim_ex} and~\ref{sec:exp_ex} with conclusion in Section~\ref{sec:conclusions}.

\section{Related Work}
\label{sec:related_work}

Existing work in robotics that deals with unknown, environmentally sparse sensors typically focus on tactile sensing~\cite{pezzementiTRO11tactileAppear, matsubara2016active,mur2015probabilistic,lepora2013active, fox2012tactile, meierTRO11probApproachTactileShape, dune2008active, lepora2013active, yi2016active, martinezRAS17activeBayesSensormotor}.
Uncertainty distributions are generated that guide the robot towards regions where the robot is likely to acquire a positive contact measurement.
Sample-based methods~\cite{pezzementiTRO11tactileAppear, matsubara2016active,mur2015probabilistic,lepora2013active, dune2008active, fox2012tactile} or maximum likelihood methods~\cite{meierTRO11probApproachTactileShape, dune2008active, lepora2013active, yi2016active, martinezRAS17activeBayesSensormotor} are used to determine the robot's next area to sample.
An issue with these methods is the omission of low probability regions.
By omitting low probability regions, only information near prior high information areas is obtained to make decisions about where to move.
Our approach uses ergodic exploration to enable a robot to systematically explore all nonzero probability regions, avoiding overly focusing on already explored regions.

Ergodic exploration has been used previously to enable search for spatially distributed information~\cite{miller2016ergodic, miller2015optimalrange, mavrommati2017eSAC}, similar to sample-based methods.
In sample-based methods, the procedure is often to sample the search space and then control the robot to the desired states calculated from the samples.
The primary benefit of using ergodic control is the ability to avoid controllers that get stuck in local minima and cyclical patterns as we show in Section~\ref{sec:comparison_gEER}. 
By using ergodic exploration, our method is able to continuously explore an environment subject to the physical constraints imposed by the robot, the sensor, and the environment. 
In particular, these physical constraints are used to localize unknown environments directly from empirically determined models.
Prior work from the authors have illustrated the use of ergodic exploration for shape estimation~\cite{abrahamRAL17}.
Here, we extend this work for a more general class of sparse non-null measurements and show that the constructed models can be used as a method for data-based environment estimation and localization.

\section{Problem Statement} 
\label{sec:problem-statement}
The goal of this work is to enable a robot to autonomously localize an environment when sensory information is sparsely distributed in the environment.
Consider the state of an environment $\sigma \in \mathcal{O}$ where $\mathcal{O}$ is the set of points that defines the surface of elements in the environment that a sensor can interact with.\footnote{For generality, we define the set of points in $\mathbb{R}^v$, but one can consider the set of points in $\mathbb{R}^3$ as the outer surface of objects that interacts with a contact sensor.}
A robot's measurements, given its state $x(t) \in \mathbb{R}^n$, is defined as
\begin{equation}\label{eq:meas_model}
y = \Gamma(\sigma, x_v) + \delta \in \mathbb{Y}^p
\end{equation}
where $\delta \sim \mathcal{N}(0, \Sigma)$ is zero mean Gaussian noise with variance $\Sigma$, $\Gamma(\sigma, s) : \mathcal{O} \times \mathbb{R}^v \to \mathbb{Y}^p$ is the measurement model, $x_v\in \mathbb{R}^v \subset \mathbb{R}^n$ is the state of the robot that interacts with the environment, and $s\in \mathbb{R}^v$ is a point in the search space.
The measurements in (\ref{eq:meas_model}) are then sparse in the environment and discrete, yielding non-null measurements in particular regions of the search space (e.g., collision measurement in the environment with a finite number of objects).
 
Given that $\Gamma(\sigma, s)$ is initially unknown for a sparse environment defined by $\sigma$, how should a robotic agent empirically construct $\Gamma(\sigma, s)$?
Moreover, provided that $\Gamma(\sigma, s)$ was constructed in some search space $\mathcal{P}$, how can the robotic agent localize within the environment?
Thus, the problem statement that this paper addresses is the following: Given some search space $\mathcal{W}$ that has the environment $\sigma$ within it, estimate the transformation $g(\theta)$ given that $\Gamma(\sigma, s)$ is initially unknown and cannot be estimated by local exploration of $\mathcal{P}$.

\begin{figure}[]
\centering
\includegraphics[scale=1]{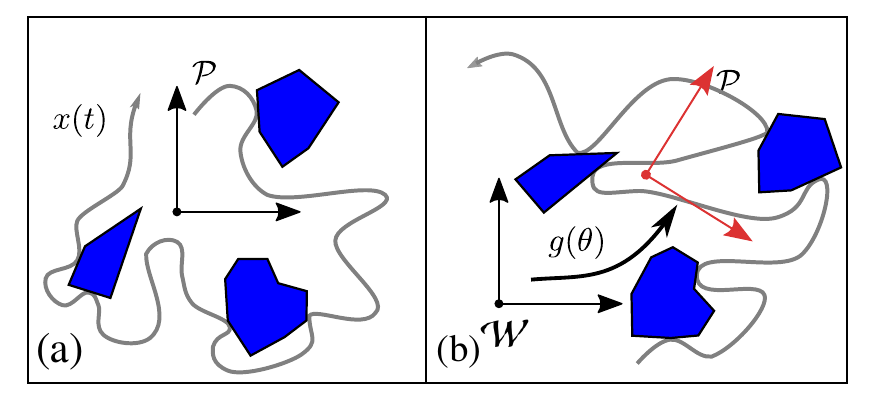} 
\caption{Objects in the environment are defined by the blue shapes. (a) A robot with an exploratory trajectory $x(t)$ is shown as the gray line estimating the measurement model of an environment in the fixed frame $\mathcal{P}$. Contact measurements are used to construct the measurement model. (b) The same environment originally shown in $\mathcal{P}$ is transformed into search space $\mathcal{W}$ under transform $g(\theta)$. The robot exploratory trajectory used to generate an estimate of the transform $g(\theta)$ is created using the data-driven measurement model in (a).}
\vspace{-5mm}
\label{fig:transf_illus}
\end{figure}

In this work we approach this problem by dividing it into two stages.
The first stage is to explore the environment and generate a data-driven measurement likelihood distribution from a set of $N$ measurements and a resulting measurement model for the environment.
This measurement likelihood distribution is necessary because the resulting distribution is not solely a sensor model, but rather a model of how the robot, environment, and the sensors interact with one another. 
The result is the construction of (\ref{eq:meas_model}), which is the measurement model for how the robot, the sensor, and the environment interact based on data.

The second stage is to use this constructed measurement model to localize and estimate the transform $g(\theta)$ from $\mathcal{P} \to \mathcal{W}$.
We illustrate the problem in Fig.~\ref{fig:transf_illus} for $\mathcal{P}, \mathcal{W} \in \mathbb{R}^2$.
Sections~\ref{sec:shape_estimation} and~\ref{sec:shape_localization} illustrate the environment measurement model construction and localization.

The following section introduces ergodicity and the ergodic metric.

\section{Ergodicity and Ergodic Control for Active Sensing}
\label{sec:erg}
In this section, we review ergodicity and the ergodic metric for active sensing and control for robotic systems~\cite{shellMult06ergodic, mathew2011metrics}.
Ergodicity and ergodic control enables us to specify the motion of the robot based on requiring that the time it spends in regions of state-space be proportional to desired spatial statistics in that region, where the desired spatial statistics are determined by an estimate of the information density.
It is later shown in Section~\ref{sec:comparison_gEER} through a comparison why using ergodic control is beneficial in the application of sparse sensing.

\subsection{Ergodicity}

Consider an agent whose state at time $t\in \mathbb{R}^+$ is $x(t) : \mathbb{R}^+ \to  \mathbb{R}^n$ and the control input to the robot as $u(t) : \mathbb{R}^+ \to \mathbb{R}^m$.
The time evolution of the robot is assumed to be governed by a control-affine dynamical system of the form
\begin{equation} \label{eq:robot_dynamics}
\dot{x} = f(x(t),u(t)) = g(x(t)) + h(x(t)) u(t)
\end{equation}
where $g(x) : \mathbb{R}^n \to \mathbb{R}^n$ is the free dynamic response of the robot, and $h(x): \mathbb{R}^n \to \mathbb{R}^{n \times m}$ is the driving dynamic response subject to input $u$.
Next, define a bounded search space domain $\mathcal{X}_v \subset \mathbb{R}^v$  whose limits are $\left[0,L_1 \right] \times \left[ 0,L_2 \right] \times \ldots \left[ 0, L_v\right]$ with $v\le n$. 
The time-averaged statistics $c(s, x(t))$ of the robot's trajectory $x(t)$ (i.e., where the robot spends most of its time) for some time interval $t \in \left[ t_i, t_i + T\right]$ is given by
\begin{equation}
c(s, x(t)) = \frac{1}{T}\int_{t_i}^{t_i+T} \delta (s - x_v(t)) dt
\end{equation}
where $\delta$ is a Dirac delta function, $T \in \mathbb{R}^+$ is the time horizon, $t_i \in \mathbb{R}^+$ is the $i^\text{th}$ sampling time, $s \in \mathbb{R}^v$ is a point in the search space, and $x_v(t) : \mathbb{R}^+ \to \mathbb{R}^v$ is the state of the robot that exists in the search space $\mathcal{X}_v$.
We then define a ``target'' distribution $\phi(s) : \mathcal{X}_v \to \mathbb{R}^+$ with respect to which the robot agent is to be ergodic (i.e., time spent during a trajectory $x(t)$ is proportionate to the spatial statistics of that region).
An ergodic metric (based on a Sobolev space norm)~\cite{mathew2011metrics} which relates the two distributions $c(s,x(t))$ and $\phi(s)$ is:
\begin{align} \label{eq:ergodic_metric}
\mathcal{E}(x(t)) & = q \,\sum_{k \in \mathbb{N}^v} \Lambda_k \left(c_k -\phi_k \right)^2   \\
& = q \, \sum_{k \in \mathbb{N}^v} \Lambda_k \left( \frac{1}{T} \int_{t_i}^{t_i + T} F_k(x_v(t)) dt - \phi_k \right)^2 \nonumber
\end{align}
where 
\begin{equation}\label{eq:get_phik}
\phi_k = \int_{\mathcal{X}_v} \phi(s) F_k(s) ds
\end{equation}
and $c_k$ are the Fourier decompositions\footnote{The cosine basis function is used here, however, any choice of basis function $F_k$ can be used.} of $c(s,x(t))$ and $\phi(s)$ with
\begin{equation*}
F_k(x) = \frac{1}{h_k}\prod_{i=1}^v \cos \left( \frac{k_i \pi x_i}{L_i} \right)
\end{equation*}
being the cosine basis function for a given coefficient $k \in \mathbb{N}^v$,  $h_k$ is the normalization factor defined in~\cite{mathew2011metrics}, $\Lambda_k = (1 + \Vert k \Vert^2)^{-\frac{v+1}{2}}$ is a weight on the frequency coefficients, and $q\in \mathbb{R}^+$ is a weight on the ergodic metric.
Equation (\ref{eq:ergodic_metric}) now relates the time spent by the robot to the spatial statistics.
A robot whose control inputs result in a trajectory $x(t)$ that minimizes (\ref{eq:ergodic_metric}) is then said to be optimally ergodic with respect to the target distribution.

\subsection{Ergodic Control}

Control inputs that result in a maximally ergodic trajectory with respect to a target distribution are synthesized following the work in~\cite{mavrommati2017eSAC}.
Instead of directly minimizing (\ref{eq:ergodic_metric}) with respect to $x(t)$ and $u(t)$, we consider the sensitivity of (\ref{eq:ergodic_metric}) with respect to an infinitesimal application duration time $\lambda \to 0 \in \mathbb{R}^+$ of the best possible control $u_\star(t) : \mathbb{R}^+ \to \mathbb{R}^m$ that sufficiently reduces (\ref{eq:ergodic_metric}) at time $\tau \in \mathbb{R}^+$ from some default control $u_\text{def}(t) : \mathbb{R}^+ \to \mathbb{R}^m$.
This sensitivity (known as the mode insertion gradient~\cite{vasudevan2013consistent, axelsson2008gradient, egerstedt2006transition, caldwell2016projection}) is given by
\begin{equation*}
\frac{\partial \mathcal{E}}{\partial \lambda} \Big \vert_\tau = \rho(\tau)^T (f_2(\tau, \tau) - f_1(\tau))
\end{equation*}
where $f_2(t, \tau) = f(x(t), u_\star(\tau))$, $f_1(t) = f(x(t), u_\text{def}(t))$, and $\rho \in \mathbb{R}^n$ is the adjoint variable which is the solution to
\begin{equation*}
\dot{\rho} = - 2 \frac{q}{T} \sum_{k \in \mathbb{N}^v} \Lambda_k (c_k - \phi_k) \frac{\partial F_k}{\partial x} - \frac{\partial f}{\partial x}^T \rho
\end{equation*}
where $\rho(t_i + T) = \bold{0} \in \mathbb{R}^n$.
The control is then
\begin{equation}
u_\star = (\Omega + R^T)^{-1} \left[ \Omega u_\text{def} + h(x)^T \rho \alpha_d \right]
\end{equation} 
where $\Omega \triangleq h(x)^T \rho \rho^T h(x)$, $\alpha_d \in \mathbb{R}^-$ parametrizes the aggressiveness of the control, and $R \in \mathbb{R}^{m \times m}$ is a positive definite matrix that weighs the control $u_\star$.
Note that as in~\cite{mavrommati2017eSAC}, saturation is taken into account and duration time $\lambda$ is found using a line search. 

Further information about the algorithm provided in Algorithm \ref{alg:eSAC} can be found in~\cite{mavrommati2017eSAC}.

\begin{algorithm}
\caption{Ergodic Control} \label{alg:eSAC}
\centering
\begin{algorithmic}[1]
% \Procedure{eSAC}{}
\State \textbf{given:} $x_i, \phi_{k,i}, t_i, T, \Delta t_\mathcal{E}$
\State  simulate $(x(t), \rho(t))$ from initial condition ($x_i, \rho_i, t_{i}, T$) and default control $u_\text{def}$
\State  compute $c_{k,i}$ using past history $\Delta t_\mathcal{E}$ and simulated trajectory $x(t)$
\State  calculate the control $u_\star (t)$ that reduces the ergodic metric with ($x(t), \rho(t), c_{k,i}, \phi_{k,i}$)
\State  $\tau \gets $ argmin $\partial \mathcal{E} / \partial \lambda$
\State saturate $u_\star(\tau)$
\State  compute $\lambda$ from line search\\
%\State $\left[\lambda_1, \lambda_2 \right] \gets$ lineSearch($\lambda$)
%\State \textbf{return} $u^*(t) \in \left[\lambda_1, \lambda_2 \right]$
\Return$u_*(\tau)$ for duration $\lambda$
% \EndProcedure
\end{algorithmic}
\end{algorithm}

The following two sections use the ergodic control policy for active sensing to first empirically determine the measurement likelihood of the environment in response to the sensor and then use the measurement model for localization.

\section{Control for Measurement Model Construction}
\label{sec:shape_estimation}

In this section, we focus on empirically constructing the measurement model of an environment from the movement of a robot and its sensor interaction with the environment.
We show how we can leverage the construction of the measurement model into the exploration procedure using ergodic control for active sensing.

\subsection{Construction of Measurement Likelihood with Active Sensing}

The approach here is similar to~\cite{abrahamRAL17} for shape estimation. 
The controller is initialized with a uniform target distribution over a finite domain (e.g., a volume in which an object resides); however, the controller can be initialized with a prior if information is known. 
The ergodic control policy then provides control input for the robot based on an initial measurement likelihood distribution (often uniform to start).
At each $t_i = t_{i-1} + t_s$ with sampling interval $t_s$, the robot collects measurements $x_i, y_i$. \footnote{Measurements can be collected asynchronously in order to prevent measurements with duration time $t_c < t_s$.}
In this work, we use a Support Vector Machine~\cite{platt1999probabilistic, wu2004probability, vapnik1998statistical}\footnote{A Gaussian kernel is used with variance parameter $\sigma^2 = 0.01$ with input as $x_i$ and output $y_i$.} to calculate $p(y_i \mid x_i, \sigma)$; however, any method\textemdash such as Gaussian processes~\cite{krause2007nonmyopic,matsubara2016active} or density estimators~\cite{silverman2018density}\textemdash can also work to approximate the measurement likelihood so long as one can generate a distribution over the measurements.
In the case of collision-based sensing, the coefficients $\phi_k$ in Eq.(\ref{eq:get_phik}) are calculated using $p(y=1 \mid x, \sigma)$ for $y=1$ yielding a collision measurement.
Since $p(y=1 \mid x, \sigma)$ is defined over the search space $s$, we numerically integrate (\ref{eq:get_phik}) for the current estimate of $p(y=1 \mid x, \sigma)$.
This allows the algorithm to leverage a particular desired measurement as an active learning procedure.  
The process for the construction of the measurement model is illustrated in Fig. \ref{shapeReconFlow} and shown in Algorithm~\ref{alg:obj_est}.

\begin{figure}[]
% \centering
\includegraphics[scale=1]{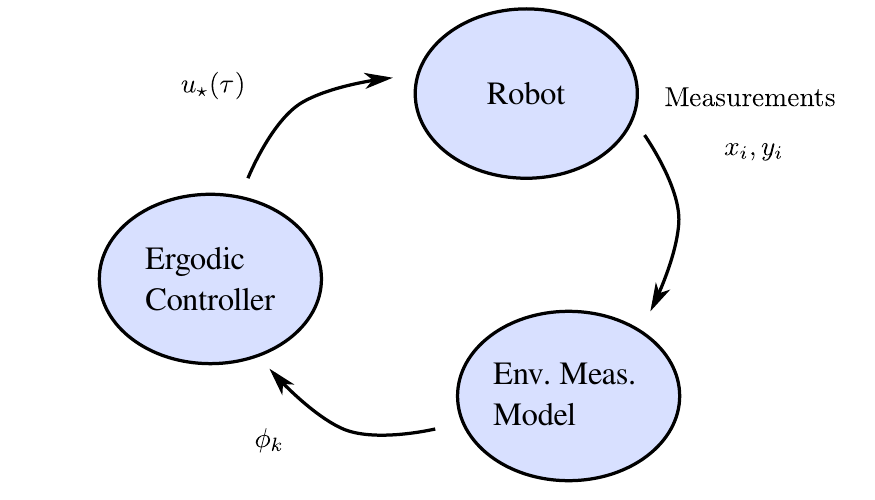}
\caption{ Block diagram for measurement model construction. Measurements are processed in the measurement likelihood probability. 
This information is then passed onto the ergodic controller which returns an active sensing policy for the robot. }
\vspace{-2mm}
\label{shapeReconFlow}
\end{figure}

\begin{algorithm}
\caption{Ergodic Control for Measurement Model Construction} \label{alg:obj_est}
\centering
\begin{algorithmic}[1]
% \Procedure{eSAC}{}
\State \textbf{define:} $x_0, \phi_{k,0}, i=0, t_0, t_f, t_s, T, \Delta t_\mathcal{E}$
\While{$t_i < t_f$}
\State compute $u_\star(\tau)$ from Algorithm~\ref{alg:eSAC}
\State sample $x_i, y_i$ from robot (assume asynchronous measurements)
\State determine $p(y_i \mid x_i, \sigma)$ from data $x_i, y_i$ using SVM (or Gaussian process)
\State calculate target distribution $\phi_{k,i}$ from $p(y \mid x, \sigma)$
\State $i \gets i +1$
\EndWhile
% \EndProcedure
\end{algorithmic}
\end{algorithm}

\section{Control for Localization using Environmental Information from Data-driven Measurement Models}
\label{sec:shape_localization}

In this section, we define the information of a environment measurement model and how one would actively sense with respect to the information of the model for localization.
Assume that we have a measurement model for an environment.
Given a new search space $\mathcal{W}\subseteq \mathbb{R}^v$ and that the model was acquired in some search space $\mathcal{P}\subseteq \mathbb{R}^v$, we can define a transform $g(\theta) \in \text{SE}(2)$\footnote{For simplicity we define the problem in the planar space, but more abstract and high dimensional transforms can be utilized.} parametrized by $\theta \in \mathbb{R}^3$ that transforms elements in $\mathcal{P}\to \mathcal{W}$ (see Fig.~\ref{fig:transf_illus} for illustration).
Localizing based on the data-driven measurement model is then reduced to identifying the parametrization $\theta$ that describes the transformation $g(\theta)$.

\vspace{-1mm}
\subsection{Information about $g\in \text{SE}(n)$}
\label{sec:eid}

In order to use ergodic control for active localization of the environment, we first need to define a distribution which captures where the robot should explore next based on the environment sensitivities with respect to its dependences.
In order to obtain this relationship, the Fisher information~\cite{fedorov2010optimal, chirikjian2000engineering, akaike1998information} of the measurement model of environment $\sigma$ with respect to changes in the search space is used.
While other measures of information such as entropy can be used, the Fisher information directly encodes parameter sensitivities that, as we will show in Lemma~\ref{lem1}, includes the Fisher information of how the robot explored the search space.
\begin{lemma}
\label{lem1}
{ \it The expected information of an environment measurement model with respect to a transformation $g(\theta)$ is given by
\begin{equation}\label{eq:target_eid}
\phi(s) = \det \left[ \int_\theta  \frac{\partial \left( g(\theta)^{-1} s \right) }{\partial \theta}^T I(s) \frac{\partial \left( g(\theta)^{-1} s \right)}{\partial \theta} p(\theta) d\theta \right]
\end{equation}
where
\begin{equation*}
I(s) =  \frac{\partial \Gamma}{\partial s}^T \Sigma^{-1} \frac{\partial \Gamma}{\partial s}
\end{equation*}
is the Fisher information with respect to the search space of the robot, $p(\theta)$ is the belief of the parameter $\theta$, and D-optimality~\cite{john1975d,fedorov2010optimal} (determinant) is chosen a measure of information.
}
\end{lemma}
\begin{proof}
We first calculate the Fisher information matrix of the empirically determined measurement model $\Gamma$ (Eq.~\ref{eq:meas_model}) with respect to the transformation:
\begin{equation} \label{eq:fish_wrt_theta}
I(s,\theta) = \frac{\partial \Gamma}{\partial \theta}^T \Sigma^{-1} \frac{\partial \Gamma}{\partial \theta}.
\end{equation}
Applying chain rule to the derivative terms in (\ref{eq:fish_wrt_theta}) gives us the expanded derivative terms
\begin{equation} \label{eq:chain_rule}
\frac{\partial \Gamma}{\partial \theta} = \frac{\partial \Gamma(\sigma, \overbrace{g(\theta)^{-1} s}^{\bar{s}} )}{\partial \bar{s}}\frac{\partial \left( g(\theta)^{-1} s \right)}{\partial \theta}.
\end{equation}
Replacing (\ref{eq:chain_rule}) into (\ref{eq:fish_wrt_theta}) gives
\begin{align}\label{eq:chain_expan}
I(s,\theta) &= \left( \frac{\partial \Gamma}{\partial \bar{s}}\frac{\partial g(\theta)^{-1} s}{\partial \theta} \right)^T \Sigma^{-1} \left(  \frac{\partial \Gamma}{\partial \bar{s}}\frac{\partial g(\theta)^{-1} s}{\partial \theta}\right) \\
& = \frac{\partial \left( g(\theta)^{-1} s \right) }{\partial \theta}^T \underbrace{ \left( \frac{\partial \Gamma}{\partial \bar{s}}^T \Sigma^{-1} \frac{\partial \Gamma}{\partial \bar{s}} \right) }_{I(s)} \frac{\partial \left( g(\theta)^{-1} s \right)}{\partial \theta} \nonumber.
\end{align}
Taking the expectation of (\ref{eq:chain_expan}) and the D-optimality measure~\cite{john1975d,fedorov2010optimal} gives us the expected information density of an environment measurement model
\begin{equation*}
\phi(s) = \det \left[ \int_\theta  \frac{\partial \left( g(\theta)^{-1} s \right) }{\partial \theta}^T I(s) \frac{\partial \left( g(\theta)^{-1} s \right)}{\partial \theta} p(\theta) d\theta \right].
\end{equation*}
\end{proof}

We can see that the Fisher information with respect to $\theta$ includes spatial information about the measurement model with respect to the search space. 
Thus, by specifying the expected information density as the target distribution for ergodic control, the robot driven towards regions where there is information about the transformation $g(\theta)$ and spatial features that are present in the measurement model.

\subsection{Estimating $p(\theta)$ using Environment Measurement Likelihood}

We can estimate $p(\theta)$ from the measurements using a Bayesian update~\cite{vapnik1998statistical, koval2015pose} of the form
\begin{equation} \label{bayes_update}
p_{i+1}(\theta) = \eta p(y_i \mid \theta, x_i, \sigma) p_i(\theta) 
\end{equation}
starting from a uniform prior $p_0(\theta)$.
Note that $p(y_i \mid \theta, x_i, \sigma)$ is simply a likelihood function as a function of $\theta$.
Therefore, we update the belief of the parameters $\theta$ using only the likelihood function which contains the family of all possible measurement models subject to transformation $g(\theta)$.
Because non-null measurements are sparse, and contact measurements are discontinuous, we use a manifold particle filter~\cite{koval2015pose,particle_filter,van2001unscented} in order to approximate (\ref{bayes_update}).
Other methods for updating $p(\theta)$ such as a Kalman filter~\cite{julier1997new} can be used; however, note that the sparsity in measurements can make the filter unstable due to the discontinuous measurement model~\cite{koval2015pose}.

The algorithm for active localization of the environment using Algorithm~\ref{alg:eSAC} is provided in Algorithm~\ref{alg:obj_local}.
An initial prior $p_0(\theta)$ and measurement likelihood function is defined.
The expected information of the environment information is calculated using finite difference.
Fourier coefficients $\phi_k$ of the expected information are computed numerically using the particles of the particle filter as a Monte Carlo integration.
Given some initial condition $x(t_0)$, the controller computes a control action subject to $\phi_k$.
Measurements are collected and used to update $p(\theta)$ using $p(y_i \mid \theta, x_i, \sigma)$.
%The procedure then repeats itself as illustrated in Fig.~\ref{++}.

\begin{algorithm}
\caption{Active Sensing for Localization using Data-driven Measurement Models} \label{alg:obj_local}
\centering
\begin{algorithmic}[1]
% \Procedure{eSAC}{}
\State \textbf{initialize:} Algorithm~\ref{alg:eSAC},$x_0, p_0(\theta), \phi_k$
\While {$t_{curr} < t_f$} 
\State calculate $u_\star(\tau)$ from Algorithm~\ref{alg:eSAC}
\State apply $u_\star(t_i)$ to robot
\State sample $x_i, y_i$ from robot (assume asynchronous measurements)
\State update $p_i(\theta)$ from $p(y_i \mid \theta_i, x_i, \sigma)$
\State compute $\phi_k$ from expected information (\ref{eq:target_eid})
%\State $\left[\lambda_1, \lambda_2 \right] \gets$ lineSearch($\lambda$)
%\State \textbf{return} $u^*(t) \in \left[\lambda_1, \lambda_2 \right]$
%\State \textbf{return} $u^*(t) \in \left[t_{curr}, t_{curr}+t_s \right]$
% \EndProcedure
\State $i \gets i +1$
\EndWhile
\end{algorithmic}
\end{algorithm}

In the following sections, we present simulated and experimental examples of constructing environment measurement model and using the models for sensing and localization of the transformed environment.

In this section, we present simulated examples of environment localization using empirically determined environment measurement models.
We first focus on examples for measurement model construction for objects in an environment using contact in $\mathbb{R}^2$ and $\mathbb{R}^3$ and then show examples of localization using a known measurement model in $\mathbb{R}^2$.
We then provide a comparison with an entropy-based method to illustrate the benefits of using ergodic control for sensors with spatially sparse information.
The experimental section then demonstrates the end-to-end algorithm where we use a robot with a collision-based sensor for localizing an transformed environment from an empirically determined environment measurement model.

\section{Simulated Examples of Environmental Measurement Model Construction and Localization}
\label{sec:sim_ex}

\subsection{Measurement Model Estimation}
Examples of environment measurement models in $\mathbb{R}^2$ and $\mathbb{R}^3$ are shown for multiple objects in Fig.~\ref{fig:2Dobject_est} and ~\ref{fig:3Dtorus}.
The dynamics of the robot are given as a double integrator system with dynamics $\dot{\bold{x}} = \left[ \dot{x}, \dot{y}, u_1, u_2 \right]^T$ in $\mathbb{R}^2$ and $\dot{\bold{x}} = \left[ \dot{x}, \dot{y}, \dot{z}, u_1, u_2, u_3 \right]^T$ in $\mathbb{R}^3$.
Here, we use $p(y=1 \mid x, \sigma)$ as the target distribution for active sensing.

The robot is initialized with a uniform distribution and the measurement likelihood is constructed as measurements are taken.
This is shown in sliced regions in $\mathbb{R}^3$  where high values of the measurement likelihood guide the motion of the robot (the same method is used when there are multiple objects in $\mathbb{R}^2$). 
Note that while measurements with higher probability are frequently acquired, the ergodic policy requires sampling from low probability regions.
This enables the robot to acquire measurements from the environment where the measurement model is well establish, refining the overall model, while incorporating measurements where the environment is uncertain. 
Thus, ergodic exploration promotes effective exploration of the environment for construction of the measurement model. 

Using this measurement model that is constructed from active exploration, we perform active localization of the environment using the resulting information of the measurement model as the target for ergodic control.
We illustrate this in the following subsection.

\begin{figure}[ht!]
\centering
\includegraphics[scale=1.0]{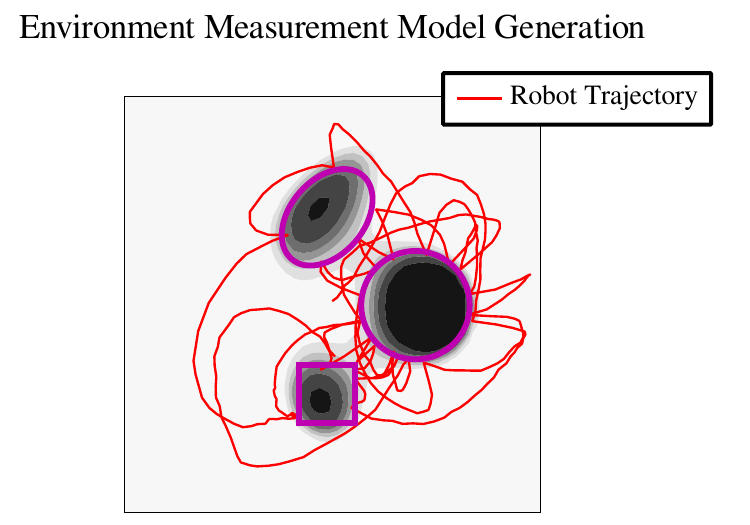}
\caption{
Estimating the measurement model of the environment with multiple objects in $\mathbb{R}^2 \subset \left[0,1 \right] \times \left[0,1 \right]$ search space using collision measurements. 
The objects are shown as with the purple outline. 
The simulation is initialized using a uniform distribution over the $\mathbb{R}^2$ search space.
The measurement likelihood from the data-driven measurement model is shown as the dark regions beneath the objects.
This distribution drives the ergodic controller as it is being created from data.
The resulting robot trajectory is shown as the red line.
Note that the robot still visits low probability region in search for new information.
Thus, the robot is able to circle around objects and better explored the environment.
} 
\label{fig:2Dobject_est}
\end{figure}

\begin{figure}[]
\centering
\includegraphics[scale=1.0]{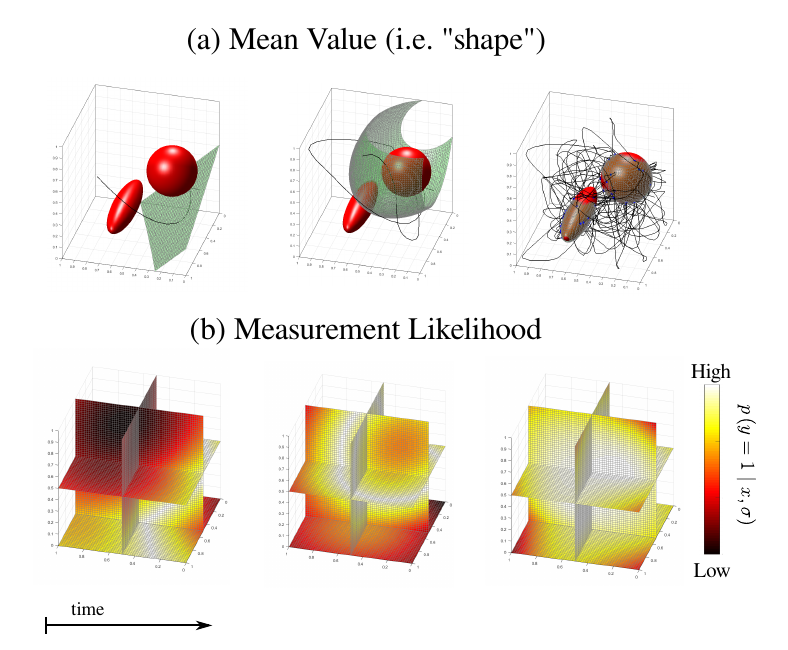}
\caption{
Estimating the measurement model of the environment with multiple objects in $\mathbb{R}^3 \subset \left[0,1 \right] \times \left[0,1 \right] \times \left[0,1 \right]$ search space using collision measurements. 
The simulation is initialized using a uniform distribution over the $\mathbb{R}^3$ search space.
(a) Mean values of the contact measurements $p(y \mid x,\sigma)$ are depicted as the green surfaces. 
Robot trajectory is shown as the black line.
(b) Slices of the measurement likelihood distribution are shown being constructed as the robot collects measurements.
Under ergodic control, the robot is able to explore the whole environment to generate the measurement model.
} 
\label{fig:3Dtorus}
\end{figure}

\begin{figure*}[ht!]
\centering
\includegraphics[scale=1]{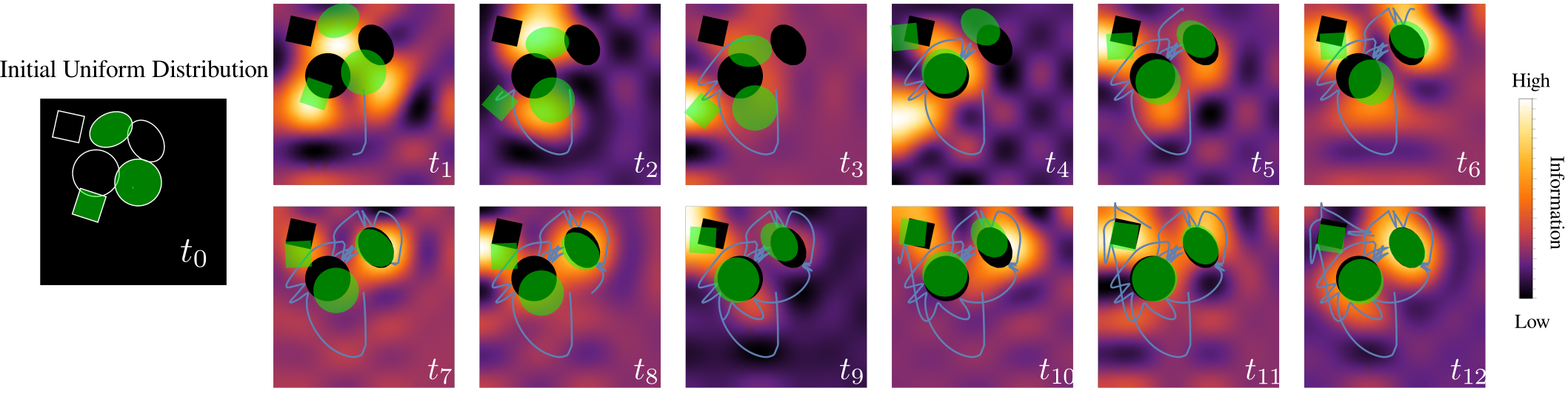}
\caption{
Localization of objects using a given measurement model is illustrated for an environment containing three objects.
The expected information of the measurement model is shown as the distribution in the background of each time frame $t_i$ space $1.25$ seconds apart.
The transparent green objects are the depiction of the current belief of their location $p(\theta)$.
The ground truth object location is shown in black and the robot trajectory in shown as the blue line. 
The robot estimates the transformation using the information of the environment measurement model as a target for active sensing.
We refer the reader to the multimedia \href{https://youtu.be/yKVlo_kE0cw}{video} for an animation of this result using a manifold particle filter~\cite{koval2015pose}.
} 
\label{many_shapes_transformation}
\end{figure*}

\subsection{Measurement Model-based Localization}

\begin{figure}[h!]
\centering
\includegraphics[scale=1]{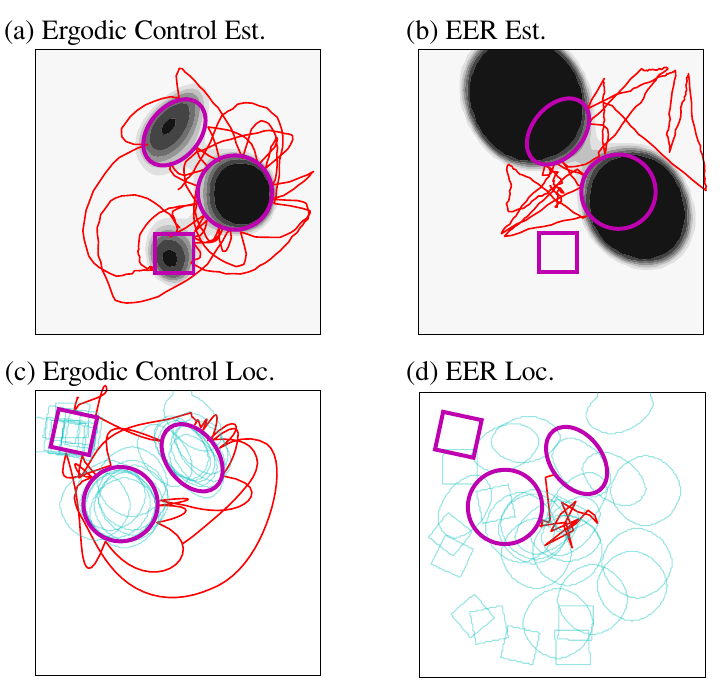}
\caption{ Comparison of Ergodic control versus entropy reduction for environment measurement estimation (a-b) and localization using the environment measurement model that is known (c-d).
Contact likelihood shown as the dark regions in (a-b).
Particle filter beliefs are shown as the cyan colored estimates of the shapes in (c-d).
Using ergodic control enables a robot to consider low probability regions which are beneficial in acquiring impact measurements that are sparse in the environment. 
Entropy reduction results in search patterns that immediately reduce the entropy which results in a lack of exploration in sparse environments. 
}
\vspace{-5mm}
\label{comparison_gEER}
\end{figure}

Simulated results for active sensing for localizing the environment from an environment measurement model are done with the double integrator model in $\mathbb{R}^2$.
In this example, three objects are in the environment: a circle, a square, and an oval as shown in Fig.~\ref{many_shapes_transformation}.
The environment measurement model is known and is used for active sensing using Algorithm~\ref{alg:obj_local}.
A uniform distribution over $\theta$ is defined as the prior.

As shown in Fig.~\ref{many_shapes_transformation}, at first contact, the belief in the environment's location starts to approach the ground truth shown as the black shapes. 
The expected information density shown in the background of Fig.~\ref{many_shapes_transformation} begins to collapse from the family of possible measurement models to a probabilistically certain measurement model.
As more measurements are aggregated, the prior on $\theta$ starts to converge.
Ergodic control then facilitates correction of the belief $p(\theta)$ via active sensing with respect to the environment information calculated from the measurement model.

\begin{figure*}[th!]
\centering
\includegraphics[scale=1]{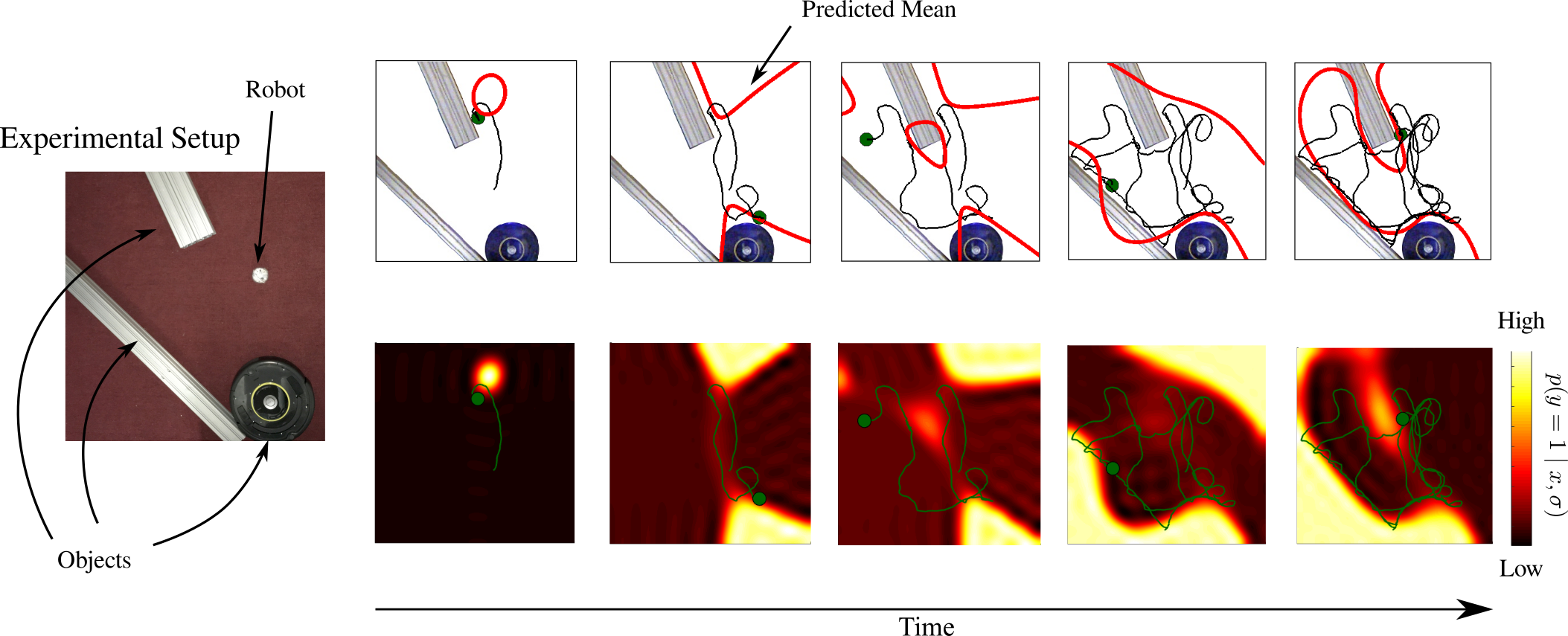} 
\caption{Left: Photo of the experiment setup of the environment state $\sigma$. The robot is initialized with a uniform target distribution. Top: Time series of the environment measurement model construction process from left to right. The red boundaries indicate the mean value of the measurement likelihood (i.e., shape). Sphero robot is shown as a green circle. The robot trajectory up to the current time is also shown. Bottom: $p(y=1 \mid x, \sigma)$ is shown being built as the robot collides with the objects in the environment.
We refer the reader to the multimedia \href{https://youtu.be/yKVlo_kE0cw}{video} for the experimental videos.
} 
\label{map_estimation}
\end{figure*}

 \begin{figure*}[ht!]
\centering
\includegraphics[scale=1.0]{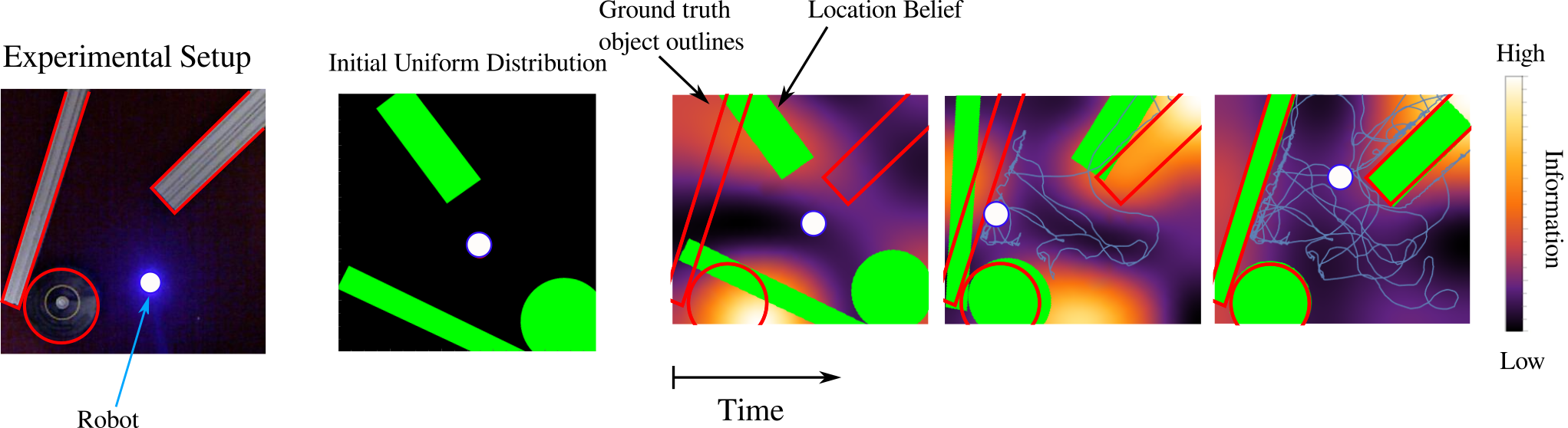} 
\caption{The robot estimates the transformation of the environment state using the measurement likelihood extracted from Fig.~\ref{map_estimation} using contact measurements. The global position of the objects is shown as the red outlines. The robot ergodic control policy result in a sequence of contact measurements that localizes the environment based on the environment information. We refer the reader to the multimedia \href{https://youtu.be/yKVlo_kE0cw}{video} for the experimental videos of the end-to-end experiment.
 }
\label{map_transformation}
\end{figure*}

\subsection{Comparison with Entropy Reduction}
\label{sec:comparison_gEER}
We compare our method using ergodic control with another method that is similar in nature known as expected entropy reduction (EER)~\cite{miller2015optimalrange, miller2016ergodic, abrahamRAL17}.
The algorithm is similar to that used in~\cite{kreucher2005sensor, fox1998active,feder1999adaptive, kreucher2007multiplatform} where a location in the search space is chosen based on where the expected change in entropy is maximized.
The location is chosen from a set of $200$ samples uniformly distributed in the search space. 
The trajectory that maximizes the entropy reduction is then chosen and an LQR controller is used to control the robot with the double integrator dynamics mentioned in the previous subsection. 

Here, the reduction in entropy is calculated as
\begin{equation*}
H(\theta) - \mathbb{E}\left[ H(\theta) \mid y(t)^+\right]
\end{equation*} 
where
\begin{equation*}
H(\theta) = - \int p(\theta) \log p(\theta) d \theta
\end{equation*}
is the entropy for a random variable $\theta$, and $y(t)^+$ is the expected measurement given the current model along the sampled trajectory.
In the case of measurement model construction with impact, $\theta$ is replaced with $y=1$ whereas in environment localization, the random variable is the transform parameters $\theta$ for $g(\theta)\in \text{SE}(2)$.

We compare a $20$ second simulation of environment construction and localization (with the known environment measurement model) in Fig.~\ref{comparison_gEER}.
Because EER prioritizes immediate reduction in entropy, only the most recent and known measurements are used.
In situations where measurements are sparse, this causes the EER algorithm to become stuck as shown in Fig.~\ref{comparison_gEER}(d). 
Similarly, in Fig.~\ref{comparison_gEER}(b), EER is only able to acquire measurements with two out of the three objects in the environment, and only one side of each object.
In contrast, using ergodic control as shown in Fig.~\ref{comparison_gEER} (a) and (c) illustrates the benefit for sensing sparse impact signals in unknown environments.
Low probability regions are not omitted and are important to both stages of our method. 
As time goes to infinity, the ergodic controller would drive the robot to explore the whole domain such that the time spent is proportional to the statistics in the region~\cite{miller2016ergodic, miller2015optimalrange, mavrommati2017eSAC, abrahamRAL17}.

In the following section, we realize our method end-to-end on a colliding robot experiment.

\section{Experimental examples of Environmental Measurement Model Construction and Localization using Sphero SPRK}
\label{sec:exp_ex}

In this section, we present experimental examples of the end-to-end algorithm for localizing an environment using an empirically determined environment measurement model.
A rolling robot driven by an internal differential drive known as the Sphero SPRK is used for collision-based sensing.
The SPRK robot has an internal contact sensor which uses its inertial measurement unit to detect a collision with an external object.
Position of the robot is acquired using OpenCV~\cite{opencv_library} and a Kinect camera.
Communication between the position tracker, controller, and robot is done through ROS~\cite{ROS}.
Once a measurement model is generated in the first stage, we use this model to localize the transformed environment.
Note that the only assumption that is made is that the robot moves subject to the double integrator dynamics in $\mathbb{R}^2$.
 
\subsection{Measurement Model Construction}
\label{sec:shape_est_exp}

Experimental results for constructing the environment measurement model using collisions with the SPRK robot are shown in Fig.~\ref{map_estimation}.
Within the first few measurements, the environment measurements are not useful as the robot has not explored the environment domain sufficiently to find impacts. 
However, due to the low probability value of unexplored regions, the robot under ergodic control must still visit these regions.
In doing so, the robot encounters more impact measurements by exploring the environment.
By the termination time ($t_f=120$ seconds), the robot has constructed an environment measurement model that is sufficient for subsequent localization.

It is worth noting that the collision sensor of the robot is sensitive to quick movements which will result in noisy measurements.
While in a engineered or physics-based approach, the noise of a robot's position and sensors are derived, our method automatically incorporates the noise effects into the measurement model.
Moreover, the underlying physical constraints in the robot and the environment are also captured in the environment measurement model.
This allows us to model the environment empirically without the need to model or estimate the environment or robot parameters.
In the following section, we demonstrate localization of the environment with respect to the information of the constructed measurement model. 

\subsection{Localization}

Using the measurement model from the previous subsection, we localize the transformed environment.
In Fig.~\ref{map_transformation}, the objects, as presented in Fig.~\ref{map_estimation}, are rotated and translated in this localization experiment. 
The initial uniform prior belief $p_0(\theta)$ is defined as a uniform distribution over the domain $\theta \in [0,1.2] \times [0, 1] \times [-2 \pi, 2 \pi]$ for the positional and rotational transformation.
Here, the experiment is run for $t_f=130$ seconds.
Once the robot begins to collect measurements, we can see that the localization of the environment starts to approach the ground truth shown as the red outlined shapes.

While collisions are important in this experiment, non-collisions are equally if not more important for localizing the environment based on the empirically determined measurement model.
Consider that collision measurements are rare occurrences in the environment that is sparse with respect to the number of objects that occupy the space.
As a consequence, the environment measurement model will be based on not having collision measurements in particular regions of space which is used in this example to further estimate the location of the transformed environment. 
We can see this in the final frame of Fig.~\ref{map_transformation} where the robot explored the regions where collision is less likely to occur.
This drives the robot to search in regions where there are no collision measurements.
In the end of the experiment, the robot is successful in estimating the current configuration ($\theta_\text{actual} = \left[ 0.5, 0.6, -1.1 \right]$, $\theta_\text{estimated} = \left[ 0.521, 0.609, -1.103 \right]$) of the environment through the use of the data-driven measurement model of the environment. 

\section{Conclusion}
\label{sec:conclusions}
We present a method for estimating measurement models characterized by spatially sparse information through active sensing and localizing with respect to these data-driven measurement models.
The ergodic control policy enables a robot to explore in sparse environments in search for new information.
We are able to show that active sensing based on data-driven environmental measurement models can be used to localize within the environment taking into account the physical constraints of the robot and the sensor physics.
Through simulated and experimental examples, it is shown that collision-based sensing, coupled with our method, is sufficient for measurement model generation and later localization.
In addition, using an ergodic control policy as the principle for active sensing is shown to direct the robot towards regions that exploit features in an environment discovered from data.
Experimental validation of the method shows that this approach can run in real-time on a robotic system.
Lastly, this method can be used for any sensor in an environment characterized by spatially sparse information by relying on motion through active sensing.

\section*{Acknowledgments}
This material is based upon work supported by the National Science Foundation under awards IIS-1426961. Any opinions, findings, and conclusions or recommendations expressed in this material are those of the author(s) and do not necessarily reflect the views of the National Science Foundation.
%% Use plainnat to work nicely with natbib. 

\balance

\bibliographystyle{plainnat}
{\small
\bibliography{references}}

\begin{thebibliography}{48}
\providecommand{\natexlab}[1]{#1}
\providecommand{\url}[1]{\texttt{#1}}
\expandafter\ifx\csname urlstyle\endcsname\relax
  \providecommand{\doi}[1]{doi: #1}\else
  \providecommand{\doi}{doi: \begingroup \urlstyle{rm}\Url}\fi

\bibitem[Abraham et~al.(2017)Abraham, Prabhakar, Hartmann, and
  Murphey]{abrahamRAL17}
I.~Abraham, A.~Prabhakar, M.~J.~Z. Hartmann, and T.~D. Murphey.
\newblock Ergodic exploration using binary sensing for nonparametric shape
  estimation.
\newblock \emph{IEEE Robotics and Automation Letters}, 2\penalty0 (2):\penalty0
  827--834, 2017.

\bibitem[Akaike(1998)]{akaike1998information}
Hirotogu Akaike.
\newblock Information theory and an extension of the maximum likelihood
  principle.
\newblock In \emph{Selected Papers of Hirotugu Akaike}, pages 199--213.
  Springer, 1998.

\bibitem[Arulampalam et~al.(2002)Arulampalam, Maskell, Gordon, and
  Clapp]{particle_filter}
M~Sanjeev Arulampalam, Simon Maskell, Neil Gordon, and Tim Clapp.
\newblock A tutorial on particle filters for online nonlinear/non-gaussian
  bayesian tracking.
\newblock \emph{{IEEE} Transactions on Signal Processing}, 50\penalty0
  (2):\penalty0 174--188, 2002.

\bibitem[Axelsson et~al.(2008)Axelsson, Wardi, Egerstedt, and
  Verriest]{axelsson2008gradient}
Henrik Axelsson, Y~Wardi, Magnus Egerstedt, and EI~Verriest.
\newblock Gradient descent approach to optimal mode scheduling in hybrid
  dynamical systems.
\newblock \emph{Journal of Optimization Theory and Applications}, 136\penalty0
  (2):\penalty0 167--186, 2008.

\bibitem[Bradski(2000)]{opencv_library}
G.~Bradski.
\newblock {The OpenCV Library}.
\newblock \emph{Dr. Dobb's Journal of Software Tools}, 2000.

\bibitem[Caldwell and Murphey(2016)]{caldwell2016projection}
TM~Caldwell and TD~Murphey.
\newblock Projection-based iterative mode scheduling for switched systems.
\newblock \emph{Nonlinear Analysis: Hybrid Systems}, 21:\penalty0 59--83, 2016.

\bibitem[Carvell and Simons(1990)]{carvell1990biometric}
George~E Carvell and DJ~Simons.
\newblock Biometric analyses of vibrissal tactile discrimination in the rat.
\newblock \emph{The Journal of Neuroscience}, 10\penalty0 (8):\penalty0
  2638--2648, 1990.

\bibitem[Chirikjian and Kyatkin(2000)]{chirikjian2000engineering}
Gregory~S Chirikjian and Alexander~B Kyatkin.
\newblock \emph{Engineering Applications of Noncommutative Harmonic Analysis:
  With Emphasis on Rotation and Motion Groups}.
\newblock CRC press, 2000.

\bibitem[Dune et~al.(2008)Dune, Marchand, Collowet, and Leroux]{dune2008active}
Claire Dune, Eric Marchand, Christophe Collowet, and Christophe Leroux.
\newblock Active rough shape estimation of unknown objects.
\newblock In \emph{{IEEE/RSJ} International Conference on Intelligent Robots
  and Systems (IROS)}, pages 3622--3627, 2008.

\bibitem[Egerstedt et~al.(2006)Egerstedt, Wardi, and
  Axelsson]{egerstedt2006transition}
Magnus Egerstedt, Yorai Wardi, and Henrik Axelsson.
\newblock Transition-time optimization for switched-mode dynamical systems.
\newblock \emph{IEEE Transactions on Automatic Control}, 51\penalty0
  (1):\penalty0 110--115, 2006.

\bibitem[Feder et~al.(1999)Feder, Leonard, and Smith]{feder1999adaptive}
Hans Jacob~S Feder, John~J Leonard, and Christopher~M Smith.
\newblock Adaptive mobile robot navigation and mapping.
\newblock \emph{The International Journal of Robotics Research}, 18\penalty0
  (7):\penalty0 650--668, 1999.

\bibitem[Fedorov(2010)]{fedorov2010optimal}
Valerii Fedorov.
\newblock Optimal experimental design.
\newblock \emph{Wiley Interdisciplinary Reviews: Computational Statistics},
  2\penalty0 (5):\penalty0 581--589, 2010.

\bibitem[Fox et~al.(2012)Fox, Evans, Pearson, and Prescott]{fox2012tactile}
Charles Fox, Mat Evans, Martin Pearson, and Tony Prescott.
\newblock Tactile {SLAM} with a biomimetic whiskered robot.
\newblock In \emph{{IEEE} International Conference on Robotics and Automation
  (ICRA)}, pages 4925--4930, 2012.

\bibitem[Fox et~al.(1998)Fox, Burgard, and Thrun]{fox1998active}
Dieter Fox, Wolfram Burgard, and Sebastian Thrun.
\newblock Active markov localization for mobile robots.
\newblock \emph{Robotics and Autonomous Systems}, 25\penalty0 (3):\penalty0
  195--207, 1998.

\bibitem[Georgakis et~al.(2017)Georgakis, Mousavian, Berg, and
  Kosecka]{georgakisRSS17objectdetection}
Georgios Georgakis, Arsalan Mousavian, Alexander Berg, and Jana Kosecka.
\newblock Synthesizing training data for object detection in indoor scenes.
\newblock In \emph{Proceedings of Robotics: Science and Systems}, 2017.
\newblock \doi{10.15607/RSS.2017.XIII.043}.

\bibitem[Gui{\'c}-Robles et~al.(1989)Gui{\'c}-Robles, Valdivieso, and
  Guajardo]{guic1989rats}
E~Gui{\'c}-Robles, C~Valdivieso, and G~Guajardo.
\newblock Rats can learn a roughness discrimination using only their vibrissal
  system.
\newblock \emph{Behavioural Brain Research}, 31\penalty0 (3):\penalty0
  285--289, 1989.

\bibitem[Heng et~al.(2015)Heng, Lee, and
  Pollefeys]{hengAutRob15calibrationVisualSLAM}
Lionel Heng, Gim~Hee Lee, and Marc Pollefeys.
\newblock Self-calibration and visual {SLAM} with a multi-camera system on a
  micro aerial vehicle.
\newblock \emph{Autonomous Robots}, 39\penalty0 (3):\penalty0 259--277, 2015.

\bibitem[Hobbs et~al.(2015)Hobbs, Towal, and Hartmann]{hobbs2015spatiotemporal}
Jennifer~A Hobbs, R~Blythe Towal, and Mitra~JZ Hartmann.
\newblock Spatiotemporal patterns of contact across the rat vibrissal array
  during exploratory behavior.
\newblock \emph{Frontiers in Behavioral Neuroscience}, 9:\penalty0 356, 2015.

\bibitem[John and Draper(1975)]{john1975d}
RC~St John and Norman~R Draper.
\newblock D-optimality for regression designs: a review.
\newblock \emph{Technometrics}, 17\penalty0 (1):\penalty0 15--23, 1975.

\bibitem[Julier and Uhlmann(1997)]{julier1997new}
Simon~J Julier and Jeffrey~K Uhlmann.
\newblock A new extension of the kalman filter to nonlinear systems.
\newblock In \emph{International Symosium on Aerospace/Defense Sensing,
  Simulation and Controls}, volume~3, pages 182--193, 1997.

\bibitem[Koval et~al.(2015)Koval, Pollard, and Srinivasa]{koval2015pose}
Michael~C Koval, Nancy~S Pollard, and Siddhartha~S Srinivasa.
\newblock Pose estimation for planar contact manipulation with manifold
  particle filters.
\newblock \emph{The International Journal of Robotics Research}, 34\penalty0
  (7):\penalty0 922--945, 2015.

\bibitem[Krause and Guestrin(2007)]{krause2007nonmyopic}
Andreas Krause and Carlos Guestrin.
\newblock Nonmyopic active learning of {G}aussian processes: an
  exploration-exploitation approach.
\newblock In \emph{International Conference on Machine Learning}, pages
  449--456, 2007.

\bibitem[Kreucher et~al.(2005)Kreucher, Kastella, and
  Hero~Iii]{kreucher2005sensor}
Chris Kreucher, Keith Kastella, and Alfred~O Hero~Iii.
\newblock Sensor management using an active sensing approach.
\newblock \emph{Signal Processing}, 85\penalty0 (3):\penalty0 607--624, 2005.

\bibitem[Kreucher et~al.(2007)Kreucher, Wegrzyn, Beauvais, and
  Conti]{kreucher2007multiplatform}
Chris Kreucher, John Wegrzyn, Michel Beauvais, and Ralph Conti.
\newblock Multiplatform information-based sensor management: an inverted uav
  demonstration.
\newblock In \emph{Defense Transformation and Net-Centric Systems 2007}, volume
  6578, page 65780Y, 2007.

\bibitem[Lepora et~al.(2013)Lepora, Martinez-Hernandez, and
  Prescott]{lepora2013active}
Nathan Lepora, Uriel Martinez-Hernandez, and Tony Prescott.
\newblock Active bayesian perception for simultaneous object localization and
  identification.
\newblock In \emph{Proceedings of Robotics: Science and Systems}, 2013.
\newblock \doi{10.15607/RSS.2013.IX.019}.

\bibitem[Martinez-Hernandez et~al.(2017)Martinez-Hernandez, Dodd, Evans,
  Prescott, and Lepora]{martinezRAS17activeBayesSensormotor}
Uriel Martinez-Hernandez, Tony~J Dodd, Mathew~H Evans, Tony~J Prescott, and
  Nathan~F Lepora.
\newblock Active sensorimotor control for tactile exploration.
\newblock \emph{Robotics and Autonomous Systems}, 87:\penalty0 15--27, 2017.

\bibitem[Mathew and Mezi{\'c}(2011)]{mathew2011metrics}
George Mathew and Igor Mezi{\'c}.
\newblock Metrics for ergodicity and design of ergodic dynamics for multi-agent
  systems.
\newblock \emph{Physica D: Nonlinear Phenomena}, 240\penalty0 (4):\penalty0
  432--442, 2011.

\bibitem[Matsubara et~al.(2016)Matsubara, Shibata, and
  Sugimoto]{matsubara2016active}
Takamitsu Matsubara, Kotaro Shibata, and Kenji Sugimoto.
\newblock Active touch point selection with travel cost in tactile exploration
  for fast shape estimation of unknown objects.
\newblock In \emph{IEEE International Conference on Advanced Intelligent
  Mechatronics (AIM)}, pages 1115--1120, 2016.

\bibitem[Mavrommati et~al.(2017)Mavrommati, Tzorakoleftherakis, Abraham, and
  Murphey]{mavrommati2017eSAC}
Anastasia Mavrommati, Emmanouil Tzorakoleftherakis, Ian Abraham, and Todd~D
  Murphey.
\newblock Real-time area coverage and target localization using
  receding-horizon ergodic exploration.
\newblock \emph{IEEE Transactions on Robotics}, 2017.

\bibitem[Meier et~al.(2011)Meier, Schopfer, Haschke, and
  Ritter]{meierTRO11probApproachTactileShape}
Martin Meier, Matthias Schopfer, Robert Haschke, and Helge Ritter.
\newblock A probabilistic approach to tactile shape reconstruction.
\newblock \emph{{IEEE} Transactions on Robotics}, 27\penalty0 (3):\penalty0
  630--635, 2011.

\bibitem[Miller and Murphey(2013)]{miller2013trajectory}
Lauren~M Miller and Todd~D Murphey.
\newblock Trajectory optimization for continuous ergodic exploration.
\newblock In \emph{{IEEE} American Control Conference (ACC)}, pages 4196--4201,
  2013.

\bibitem[Miller and Murphey(2015)]{miller2015optimalrange}
Lauren~M Miller and Todd~D Murphey.
\newblock Optimal planning for target localization and coverage using range
  sensing.
\newblock In \emph{{IEEE} International Conference on Automation Science and
  Engineering (CASE)}, pages 501--508, 2015.

\bibitem[Miller et~al.(2016)Miller, Silverman, MacIver, and
  Murphey]{miller2016ergodic}
Lauren~M Miller, Yonatan Silverman, Malcolm~A MacIver, and Todd~D Murphey.
\newblock Ergodic exploration of distributed information.
\newblock \emph{{IEEE} Transactions on Robotics}, 32\penalty0 (1):\penalty0
  36--52, 2016.

\bibitem[Mitchinson et~al.(2007)Mitchinson, Martin, Grant, and
  Prescott]{mitchinson2007feedback}
Ben Mitchinson, Chris~J Martin, Robyn~A Grant, and Tony~J Prescott.
\newblock Feedback control in active sensing: rat exploratory whisking is
  modulated by environmental contact.
\newblock \emph{Proceedings of the Royal Society of London B: Biological
  Sciences}, 274\penalty0 (1613):\penalty0 1035--1041, 2007.

\bibitem[Mur-Artal and Tard{\'o}s(2015)]{mur2015probabilistic}
Ra{\'u}l Mur-Artal and Juan~D Tard{\'o}s.
\newblock Probabilistic semi-dense mapping from highly accurate feature-based
  monocular {SLAM}.
\newblock \emph{Proceedings of Robotics: Science and Systems}, 2015.

\bibitem[Nuger and Benhabib(2016)]{nuger2016multicamera}
Evgeny Nuger and Beno Benhabib.
\newblock Multicamera fusion for shape estimation and visibility analysis of
  unknown deforming objects.
\newblock \emph{Journal of Electronic Imaging}, 25\penalty0 (4):\penalty0
  041009--041009, 2016.

\bibitem[Pezzementi et~al.(2011)Pezzementi, Plaku, Reyda, and
  Hager]{pezzementiTRO11tactileAppear}
Zachary Pezzementi, Erion Plaku, Caitlin Reyda, and Gregory~D Hager.
\newblock Tactile-object recognition from appearance information.
\newblock \emph{{IEEE} Transactions on Robotics}, 27\penalty0 (3):\penalty0
  473--487, 2011.

\bibitem[Platt(1999)]{platt1999probabilistic}
John Platt.
\newblock Probabilistic outputs for support vector machines and comparisons to
  regularized likelihood methods.
\newblock \emph{Advances in Large Margin Classifiers}, 10\penalty0
  (3):\penalty0 61--74, 1999.

\bibitem[Quigley et~al.(2009)Quigley, Conley, Gerkey, Faust, Foote, Leibs,
  Wheeler, and Ng]{ROS}
Morgan Quigley, Ken Conley, Brian~P. Gerkey, Josh Faust, Tully Foote, Jeremy
  Leibs, Rob Wheeler, and Andrew~Y. Ng.
\newblock Ros: an open-source robot operating system.
\newblock In \emph{ICRA Workshop on Open Source Software}, 2009.

\bibitem[Rasshofer and Gresser(2005)]{rasshofer2005automotive}
RH~Rasshofer and K~Gresser.
\newblock Automotive radar and lidar systems for next generation driver
  assistance functions.
\newblock \emph{Advances in Radio Science}, 3\penalty0 (B. 4):\penalty0
  205--209, 2005.

\bibitem[Schwarz(2010)]{schwarz2010lidar}
Brent Schwarz.
\newblock Lidar: Mapping the world in 3d.
\newblock \emph{Nature Photonics}, 4\penalty0 (7):\penalty0 429, 2010.

\bibitem[Shell et~al.(2005)Shell, Jones, and Matari{\'c}]{shellMult06ergodic}
Dylan~A Shell, Chris~V Jones, and Maja~J Matari{\'c}.
\newblock Ergodic dynamics by design: A route to predictable multi-robot
  systems.
\newblock In \emph{Multi-Robot Systems. From Swarms to Intelligent Automata
  Volume III}, pages 291--297. Springer, 2005.

\bibitem[Silverman(2018)]{silverman2018density}
Bernard~W Silverman.
\newblock \emph{Density estimation for statistics and data analysis}.
\newblock Routledge, 2018.

\bibitem[Van Der~Merwe et~al.(2001)Van Der~Merwe, Doucet, De~Freitas, and
  Wan]{van2001unscented}
Rudolph Van Der~Merwe, Arnaud Doucet, Nando De~Freitas, and Eric~A Wan.
\newblock The unscented particle filter.
\newblock In \emph{Advances in Neural Information Processing Systems}, pages
  584--590, 2001.

\bibitem[Vapnik and Vapnik(1998)]{vapnik1998statistical}
Vladimir~Naumovich Vapnik and Vlamimir Vapnik.
\newblock \emph{Statistical Learning Theory}, volume~1.
\newblock Wiley New York, 1998.

\bibitem[Vasudevan et~al.(2013)Vasudevan, Gonzalez, Bajcsy, and
  Sastry]{vasudevan2013consistent}
Ramanarayan Vasudevan, Humberto Gonzalez, Ruzena Bajcsy, and S~Shankar Sastry.
\newblock Consistent approximations for the optimal control of constrained
  switched systems---part 1: A conceptual algorithm.
\newblock \emph{SIAM Journal on Control and Optimization}, 51\penalty0
  (6):\penalty0 4463--4483, 2013.

\bibitem[Wu et~al.(2004)Wu, Lin, and Weng]{wu2004probability}
Ting-Fan Wu, Chih-Jen Lin, and Ruby~C Weng.
\newblock Probability estimates for multi-class classification by pairwise
  coupling.
\newblock \emph{Journal of Machine Learning Research}, 5:\penalty0 975--1005,
  2004.

\bibitem[Yi et~al.(2016)Yi, Calandra, Veiga, van Hoof, Hermans, Zhang, and
  Peters]{yi2016active}
Zhengkun Yi, Roberto Calandra, Filipe Veiga, Herke van Hoof, Tucker Hermans,
  Yilei Zhang, and Jan Peters.
\newblock Active tactile object exploration with gaussian processes.
\newblock In \emph{{IEEE/RSJ} International Conference on Intelligent Robots
  and Systems (IROS)}, pages 4925--4930, 2016.

\end{thebibliography}

\end{document}